\documentclass[letterpaper, 10 pt, conference]{IEEEtran}
\IEEEoverridecommandlockouts

\usepackage[numbers]{natbib}

\usepackage{microtype}
\usepackage{graphicx}
\usepackage{booktabs}

\usepackage{hyperref}

\usepackage[noend]{algorithmic}
\usepackage{algorithm}

\let\labelindent\relax
\usepackage{enumitem}

\usepackage{amsthm}
\usepackage{mathtools,fancyhdr,amsmath,amssymb}

\usepackage{amsfonts}
\usepackage{textcomp}
\usepackage[utf8]{inputenc}
\theoremstyle{definition}
\usepackage[footnotesize]{caption}
\usepackage{subcaption}
\usepackage{bm}
\newtheorem{definition}{Definition}[section]
\newtheorem{assumption}{Assumption}[section]
\theoremstyle{theorem}
\newtheorem{theorem}{Theorem}[section]

\setlist{
	leftmargin=4.2mm
} 

\newcommand{\subparagraph}{}
\usepackage{titlesec}
\titlespacing*{\subsection}{0pt}{0.15\baselineskip}{0.05\baselineskip}
\titlespacing*{\section}{0pt}{0.6\baselineskip}{0.5\baselineskip}

\setlist[itemize]{topsep=0pt, partopsep=0pt}
\newcommand{\argmin}{\operatornamewithlimits{arg\,min}}
\newcommand{\argmax}{\operatornamewithlimits{arg\,max}}

\DeclarePairedDelimiter\norm{\lVert}{\rVert}

\begin{document}
%
\title{\LARGE \bf Multi-Agent Safe Planning with Gaussian Processes}

\author{Zheqing Zhu$^{1}$, Erdem B\i y\i k$^{2}$ and Dorsa Sadigh$^{2,3}$%
	\thanks{Emails: \{{\tt\small zheqzhu}, {\tt\small ebiyik}, {\tt\small dorsa}\}{\tt\small @stanford.edu}}%
	\thanks{$^{1}$Management Science \& Engineering, Stanford University, CA, USA}%
	\thanks{$^{2}$Electrical Engineering, Stanford University, CA, USA}%
	\thanks{$^{3}$Computer Science, Stanford University, CA, USA}%
}

\maketitle

\begin{abstract}
Multi-agent safe systems have become an increasingly important area of study as we can now easily have multiple AI-powered systems operating together. In such settings, we need to ensure the safety of not only each individual agent, but also the overall system. In this paper, we introduce a novel multi-agent safe learning algorithm that enables decentralized safe navigation when there are multiple different agents in the environment. This algorithm makes mild assumptions about other agents and is trained in a decentralized fashion, i.e. with very little prior knowledge about other agents' policies. Experiments show our algorithm performs well with the robots running other algorithms when optimizing various objectives.
\end{abstract}


%
\IEEEpeerreviewmaketitle

\section{Introduction}
Safety in multi-agent systems is vital in collaborative tasks. Even when the agents can observe each other and know the dynamics of their environment, operating in a decentralized manner without knowing each other's policy makes it very challenging to guarantee safety. While it is a difficult problem, such safety-critical systems with multiple agents are commonly observed in autonomous driving \cite{cao2013overview,fisac2019hierarchical}, collaborative quadrotor control \cite{viseras2016decentralized,bahnemann2017decentralized,wang2018cooperative}, multi-robot coordination \cite{amato2015probabilistic,amato2015planning,omidshafiei2015decentralized}. Ensuring the safety in these environments is extremely critical as the failures can cause damage not only to the agents themselves, but also to the environment. 

We specifically study tasks where state-action pairs of the agents, and dynamics of the system are known, but agents are decentralized, i.e. they do not know other agents' policies. We decompose safety into two parts as \emph{individual} and \emph{joint} safety. 
There are many real-world use cases that fit into this framework. Consider a Mars exploration task where multiple rovers explore a region together, as visualized in Fig.~\ref{fig.frontfig}, without explicit communication for increasing energy efficiency and avoiding delays due to communication. Individual safety could be each rover avoiding environment obstacles or steep terrain, and joint safety might refer to keeping the distance between rovers in a specific range to avoid collisions. The setting also works for autonomous cars that are trying to avoid collisions while optimizing each car's speed, and competitive robot teams where two teams are competing in a game setting, e.g. soccer game, but do not want to collide with one another to cause fatal damage.

\begin{figure}[h]
	\centering
	\includegraphics[width=0.8\columnwidth]{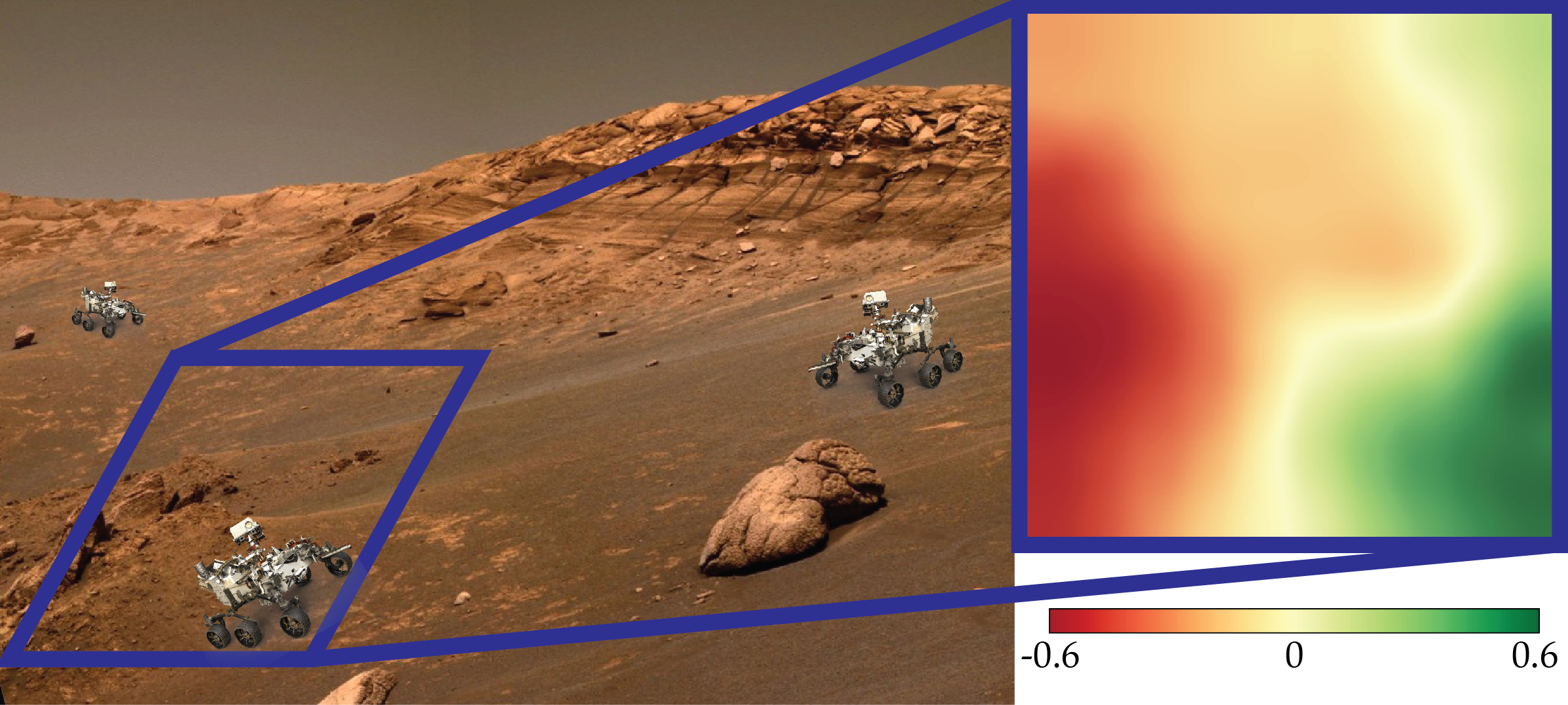}
	\caption{Representative figure for three Mars rovers exploring a region on the surface of Mars. Colors represent the normalized altitudes of the terrain such that low altitude regions shown as red are unsafe for the rovers. It is also unsafe if two rovers operate in the same region due to the risk of collisions.}
	\vspace{-10px}
	\label{fig.frontfig}
\end{figure}

While assuming a centralized controller enables us to formulate the problem as a single-agent problem, it is not realistic in practice, because such systems do not scale well with the number of agents in the environment: both state and action spaces grow exponentially with the number of agents. On the other hand, as opposed to many prior works, most real-world tasks require continuous state spaces. Therefore, we would like to \emph{enable multiple decentralized agents to operate in an environment with a continuous state space without getting into individually or jointly unsafe states}.

Such a problem would naturally fit into a relaxed subset of decentralized partially observable Markov Decision Process (Dec-POMDP) framework \cite{amato2013decentralized} where the relaxation is due to full observability and independence of agents. However, not knowing other agents' policies makes the problem difficult. We attempt to learn other agents' policies, which eases the objective of avoiding jointly unsafe situations. Specifically, we use Gaussian Processes (GP) both for estimating other agents' actions, which then enables us to increase the probability joint safety, and for modeling the individual safety. We model the risk using confidence bounds, again both for other agents' possible actions, and for individual safety.

Our contributions in this paper are the following:
\begin{itemize}[nosep]
	\item We develop a novel decentralized multi-agent planning algorithm on continuous state space to achieve overall system safety more often than the existing algorithms.
	\item We show our algorithm has linear time complexity on the actions space size and polynomial time complexity on the number of visited states.
	\item Experiments show robots with our algorithm safely collaborate for exploration and exploitation with agents running standard planning and reinforcement learning algorithms.
\end{itemize}

\section{Related Work}
\textbf{Safe Exploration.}
Single-agent safe exploration has been extensively studied. \citet{turchetta2016safe} established the \textsc{SafeMDP} algorithm for deterministic Markov Decision Processes (MDP) with discrete state spaces by assuming the risk value of each state is under some regularity that allows the use of GPs. \citet{wachi2018safe}, again employing GPs, extended the work to both exploration and exploitation. Using a similar idea, \citet{berkenkamp2016safe} established a parameter exploration algorithm under multiple constraints, that would safely tune robots’ parameters, as an extension to earlier work \cite{sui2015safe}. More recently, \citet{biyik2019efficient} leveraged continuity assumptions to deterministically guarantee safety for efficient exploration in unknown environments. \citet{bajcsy2019efficient} and \citet{fridovich2019safely} developed reachability-based frameworks for safe navigation, again in unknown environments. While we employ many similar ideas, all of these works focused only on single-agent settings, whereas ensuring safety in decentralized multi-agent systems require modeling the other agents in the environment.

\noindent\textbf{Safe Reinforcement Learning.}
On the safe reinforcement learning (RL), \citet{basu2008learning} developed a learning algorithm to handle risk-sensitive cost. \citet{geibel2005risk} formulated the risk as a second criterion based on cumulative return. \citet{moldovan2012safe} proposed an algorithm that constrains the attention to the guaranteed safe policies. \citet{fisac2018general} proposed a framework using reachability methods for guaranteeing safety during learning. \citet{berkenkamp2017safe} developed a framework for model-based RL using Lyapunov stability verification. However these works, too, focused only on single-agent settings. We refer to \cite{garcia2015comprehensive} for a comprehensive survey on safe RL.

\noindent\textbf{Multi-Agent Reinforcement Learning.}
Many recent works studied how to train multiple robots that will operate in a decentralized manner \cite{da2018autonomously}. They developed various techniques and got successful results in several different tasks, such as simulated navigation \cite{lowe2017multi}, video games \cite{foerster2018counterfactual}, target tracking \cite{zhang2013coordinating}, soccer \cite{le2017coordinated}, etc. However, these tasks are either not safety-critical, or safety is hard-coded, which requires careful analysis and design. \citet{fisac2015reach} proposed an extension of Hamilton-Jacobi methods on reach-avoidance problem. While their approach specifies conventions agents typically follow, we make only very weak assumptions about other agents.

\noindent\textbf{Intent Inference in Multi-Agent Settings.}
Predicting other agents' actions has been studied in the context of intent inference (and theory of mind \cite{devin2016implemented,hellstrom2018understandable}). \citet{bai2015intention} demonstrated it is possible to autonomously drive in a crowd by estimating the intentions of pedestrians. \citet{sadigh2016information} developed a method to enable the learning of other agents' internal states by actively probing them. Both of these works rely on the assumptions about the models of other agents' policies. Recently, \citet{song2018multi} proposed a multi-agent extension of generative adversarial imitation learning, which can help learn other agents' policies after observing a few instances. However, this is mostly limited to offline settings as it requires large computation powers to learn the policies.

\section{Problem Definition}
In this section, we formalize the multi-agent safe exploration problem and our key assumptions. In our setting, multiple robots interact with each other by \emph{simultaneously} taking actions that explore a shared environment. Each robot should take only safe actions that not only satisfy the individual safety constraints, but also avoid moving the overall system to a jointly undesirable configuration.

The key challenge is that each robot needs to act simultaneously in a decentralized manner. Without prior knowledge of other agents' policies, they make decisions based on their own policies after observing the previous states and actions of all the agents. \emph{Our goal is to develop such a decentralized strategy to safely navigate in the environment.} 
 


We model this system as a Markov Decision Process with multiple agents (MDP-MA), where the agents share the same environment but their transitions are factorized.

\begin{definition}\label{mdp_def}
\textbf{(MDP-MA)} An MDP-MA with $N$ agents is defined by a tuple $(\mathcal{S}, \mathcal{A}, f, r)$. $\mathcal{S}$ is a continuous set of states, where $\mathbf{s}_t = (s^1_t, s^2_t, ..., s^N_t) \in \mathcal{S}^N$ represents the state of all $N$ agents at time $t$.
We note the state of each agent $s_t^i$ lies in $\mathcal{S}$.
Similarly, $\mathcal{A}$ is the discrete set of actions, i.e., $\mathbf{a}_t = (a^1_t, a^2_t, ..., a^N_t) \in \mathcal{A}^N$. $f$ is a probability distribution such that $f(s_{t+1}|s_t,a_t)$ is the probability of reaching $s_{t+1}$ from $s_t$ with action $a_t$. All agents act synchronously, so $\mathbf{s}_{t+1} \sim f(\cdot | \mathbf{s}_t,\mathbf{a}_t) = [f(\cdot|s^1_t,a^1_t), \dots, f(\cdot|s^N_t,a^N_t)]$. Our formulations can be generalized to the settings where the action spaces are state- or agent-dependent.
$r: \mathcal{S} \to \mathbb{R}$ is the unknown reward function shared by all the agents. In terms of rewards, agent $i$ can only observe $r(s^i_t) + w_t^i$ at time step $t$, where $w_t^i \sim \mathcal{N}(0, {\eta^i}^2)$. 
\end{definition}
\vspace{-5px}

The goal of each agent is to take actions that optimize its own reward while \emph{safely} planning in the environment. We now formalize the notion of \emph{safety} in this MDP-MA.
\begin{definition}\label{safe_def} \textbf{(Safety)}
A state $s$ is \emph{individually safe} if and only if $r(s) \geq h$, for some safety threshold $h$. In addition to individual safety, a set $U\subset S^N$ defines \emph{jointly unsafe states}.
\end{definition}
\vspace{-5px}
We want to note two important points. First, it is possible to have $(s^1_t, \dots, s^N_t)\in U$, even though $r(s^i_t) \geq h$ for all $i \in \{1,\dots, N\}$. For example, a specific location might be individually safe for drones, but having multiple drones in that location might cause catastrophic collisions. Second, as the reward is a function of individual states, jointly unsafe states are not induced by the reward function.

\begin{assumption} \textbf{(Observability Assumptions)}
Following \cite{turchetta2016safe} and \cite{wachi2018safe}, we assume:
\begin{itemize}[nosep]
	\item All agents know the set $U$, the safety threshold $h$, the initial state $\mathbf{s}_0\! \in\! S^N\!\setminus\! U$, and the dynamics $f$.
	\item At any time step $t$, all agents observe the states $\mathbf{s}_t$ and the actions $\mathbf{a}_t$.
\end{itemize}
\end{assumption}
\vspace{-5px}
Since $f$, the transition distribution, is known at all times to all agents, any agent could estimate the next state $\mathbf{s}_{t+1}$ based on its information about the MDP-MA if it could predict the actions of other agents accurately. However, the agents do not have prior information about other agents' policies.

\begin{assumption}\label{GP_assumption} \textbf{(Reward Function Assumptions)} With no assumption on $r$, an agent cannot learn other agents' policies without repeatedly observing all possible state configurations. Therefore,
\begin{itemize}[nosep]
	\item We assume that $\mathcal{S}$ is endowed with a positive definite kernel function $k_{r}(s, s')$ and that $r(s)$ has bounded norm in the associated Reproducing Kernel Hilbert Space (RKHS).
	\item We also assume $L$-Lipschitz continuity of the reward function $r$ with respect to some metric $d(\cdot, \cdot)$ on $\mathcal{S}$. This is guaranteed by many commonly used kernels with high probability \cite{ghosal2006posterior, srinivas2010gaussian}. 
\end{itemize}  
\end{assumption}
\vspace{-5px}
\textbf{Objective.} Given the problem definition and assumptions, our goal is to achieve a predefined objective, e.g. maximizing the number of states explored or maximizing cumulative reward, while avoiding individually and jointly unsafe states.

\section{Multi-Agent SafeMDP Algorithm}
We formalize individual and joint safety separately and use confidence bounds to determine whether or not a constraint is satisfied. Sec.~\ref{iterative_safemdp} introduces a method to find the states that satisfy individual safety with high confidence and Sec.~\ref{agent_model} explains our approach for modeling other agents' policies to achieve joint safety.

\subsection{Iterative SafeMDP}\label{iterative_safemdp}
Our algorithm finds the most desirable one-step reachable state with high probability of being individually safe, as well as being returnable to the previously found safe states within one step.

We use a Gaussian Process (GP) to model the reward function on the state space for the agent running our algorithm \cite{srinivas2010gaussian}. Given Assumption~\ref{GP_assumption}, we model the reward function using a GP as
\begin{align}
r \sim \mathcal{GP}_r(\mu_r, k_r),
\end{align}
where $\mu_r(s)$ is the mean function and $k_r(s, s')$ is the covariance function. We assume the prior mean $\mu_r$ of all states is $0$. The variance $\sigma^2_r(s)=k_r(s,s)$ encodes the noise of the environment from observations and our uncertainty. Our algorithm updates $\mathcal{GP}_r$ after every reward observation and utilizes it to estimate the reward of the next state.  We refer to \cite{rasmussen2004gaussian} for GP posterior update with new observations. We define the confidence bounds of $\mathcal{GP}_r$ as ${C_r}_t(s) = [{\mu_r}_t(s) \pm \beta_r(t) {\sigma_r}_t(s)]$ for some $\beta_r(t)\geq0$. We denote 
\begin{align}
\begin{split}
r^{\mu}_t(s) &= \mu_{r_t}(s),\\ 
\quad r^u_t(s) &= \mu_{r_t}(s) + \beta_r(t) {\sigma_r}_t(s),\\ 
\quad r^l_t(s) &= \mu_{r_t}(s) - \beta_r(t) {\sigma_r}_t(s). \\
\end{split}
\end{align} 

Given a state-action pair at time step $t$, the lower bound on the expected reward of the next state is:
\begin{align}\label{reward}
lr_t(s, a) = \int_{\mathcal{S}} r^l_t(\varsigma) f(\varsigma|s, a) d\varsigma.
\end{align} 

\begin{definition} \textbf{(Returnability)}
	\label{returnability}
At any time step $t$, given an initial safe state set $S_0\supseteq\{s_0\}$, we define $S_j = S_{j-1} \cup \{s \in S \;|\; \exists a \in A, \: \int_{\varsigma\in S_{j-1}} f(\varsigma|s,a)d\varsigma \geq \tau \:\land \:lr_t(s, a) \geq h  \}$ and $\bar{S}_t = \lim_{j\to\infty} S_j$, for some $\tau\in[0,1]$. And the returnability of a state-action pair $(s,a)$ is defined as
\begin{align}\label{return}
return_t(s, a) = \int_{\bar{S}_t} f(\varsigma|s, a)d\varsigma.
\end{align}
\end{definition}
We emphasize that $\tau$ is a threshold on the returnability property.
Given the lower bound on the expected reward and the Definition~\ref{returnability}, we say a state-action pair $(s_t, a_t)$ can be safely realized if and only if $lr_t(s_t, a_t) \geq h$ and $return_t(s_t, a_t) \geq \tau$. We use this definition to restrict the actions our agent can take.

\subsection{Multi-agent Modeling}\label{agent_model}
To avoid jointly unsafe states, our algorithm predicts other agents' actions. To do so, we make the following assumption to account for both exploration and exploitation \cite{tijsma2016comparing}.

\begin{assumption} \label{explore_assumption} \textbf{(General Policy Assumptions)}
	Agent $i$ (any other agent in the environment) follows a policy that is a combination of commonly used exploration strategies, namely Optimism in the face of Uncertainty (OFU) and Boltzmann policy, which we describe in detail below. This is a mild assumption, because combining these strategies leads to a very general and inclusive class of policies.
\end{assumption}

We define a function $g^i$, which given the combined policy, computes the probability of taking an action that would transition from $s^i_t$ to $s^i_{t+1}$. We now describe the elements of the combined policy.

\noindent \textbf{Q-Function with GP.} 
Before we derive OFU and Boltzmann, we first introduce the $Q$-functions (action-value functions) for each agent. Given the GP reward model, we leverage $Q$-learning to find the mean of the $Q$-function:
\begin{align}\label{V_update}
Q^{\mu}_t(s, a) = \int_{\mathcal{S}} \left(r^{\mu}_t(\varsigma) + \gamma \max_{a'}Q^{\mu}_t(\varsigma, a')\right) f(\varsigma|s, a) d\varsigma
\end{align}
where $\gamma\in(0,1)$ is a discount factor. We similarly define $Q^u$ and $Q^l$ as the confidence bounds of the $Q$-function by using $r^u$ and $r^l$ instead of $r^{\mu}$, respectively. We denote $C_{Q_t}(s, a) = [Q_t^l(s, a), Q_t^u(s,a)]$. These definitions rely on an independence assumption for faster computation and good results both empirically and theoretically (see Sec.~\ref{sec:theory} and \ref{sec:experiments}). In fact, the posterior distribution above assuming independence is more conservative due to reward correlation between states. 

To learn a $Q$-function in a continuous domain, we adapt temporal difference error (TD-error) learning. For $Q^\mu$, $Q^u$ and $Q^l$, which are parameterized by $\theta^\mu$, $\theta^u$ and $\theta^l$, respectively,
\begin{align}
\theta \leftarrow \argmin_\theta \norm{r_t^{\mu}(s') + \max_{a'}Q_{\theta^{-}}(s', a') - Q_{\theta}(s_t, a_t)}
\end{align}
where $s'\sim f(\cdot|s_t,a_t)$, and $\theta^-$ is the set of parameters that are updated after every $\Delta$ time steps by copying $\theta$ (as in \cite{mnih2015human}). In this way, $Q$ values can be approximated.

\noindent\textbf{Optimism in the face of Uncertainty.}
OFU is a classic exploration strategy that favors the actions with high upper bound in potential return \cite{lai1985asymptotically,agrawal1995sample}. Given a Q-function distribution, OFU can be written as
\begin{align}
\pi_o(s, a) = \frac{\exp(Q^u(s, a) / T_o)}{\sum_{a'}\exp(Q^u(s, a')/ T_o)},
\end{align}
which outputs the probability of taking action $a$ at state $s$, where $T_o>0$ is the unknown temperature parameter. To derive an upper bound of the probability that a state-action pair is observed, we get the upper bound of $\pi^{i}_o$ as 
\begin{align}
\begin{split}
\pi^{iu}_o(s, a) &= \frac{\exp(Q^{iu}(s, a) / T^i_o)}{\exp(Q^{iu}(s, a) / T^i_o) + \sum_{a' \neq a}\exp(Q^{i\mu}(s, a')/ T^i_o)}\\ 
&\geq \pi^i_o(s, a)\\
\end{split}
\raisetag{1\normalbaselineskip}
\end{align}

\noindent\textbf{Boltzmann Policy.}
Similar to OFU, the Boltzmann exploration strategy \cite{barto1991real} is:
\begin{align}\label{boltzmann}
    \pi_b(s, a) = \frac{\exp(Q^{\mu}(s, a) / T_b)}{\sum_{a'}\exp(Q^{\mu}(s, a')/ T_b)}
\end{align}
where $T_b > 0$ is the unknown temperature parameter. The upper confidence bound is:
\begin{align}
\begin{split}
\pi^{iu}_b(s, a) &= \frac{\exp(Q^{i\mu}(s, a) / T^i_b)}{\exp(Q^{i\mu}(s, a) / T^i_b) + \sum_{a' \neq a}\exp(Q^{il}(s, a')/ T^i_b)}\\
& \geq \pi^i_b(s, a).\\
\end{split}
\raisetag{1\normalbaselineskip}
\end{align}

Exploitation is implicitly covered within the Boltzmann policy formulation with a right choice of $T_b$. 

\noindent\textbf{Combining OFU and Boltzmann.} The agent $i$ will follow OFU with probability $\epsilon^i$, and Boltzmann with probability $1-\epsilon^i$ for some unknown $0\leq\epsilon^i\leq 1$. When $T^i_o \rightarrow \infty$ and $T^i_b \rightarrow 0$, the policy reduces to a pure $\epsilon$-greedy policy. This combination completes the definition of $g^i$, the transition probability between states given estimated policies and $\epsilon^i$ (see Assumption~\ref{explore_assumption}), and allows it to model all of OFU, Boltzmann and $\epsilon$-greedy strategies. 

\subsection{Inference of Joint Policy Parameters}
 \label{parameter_sec}
Having described the strategies, we now explain how to jointly estimate $T^i_o, T^i_b$ and $\epsilon^i$. By assuming uniform prior over $P(\epsilon^i, T^i_b, T^i_o|Q^i)$, after a series of observations $\xi^i$, Bayes' rule gives
\begin{align}
    P(\epsilon^i, T^i_b, T^i_o | \xi^i, Q^i) &\propto P(\xi^i | \epsilon^i, T^i_b, T^i_o, Q^i).
\end{align}
We find the maximum likelihood estimate of $(\epsilon^{i*}, T^{i*}_b, T^{i*}_o)$:
\begin{align}
    \epsilon^{i*}, T^{i*}_b, T^{i*}_o = \argmax_{\epsilon^i, T^i_b, T^i_o} \log(p(\xi^i | \epsilon^i, T^i_b, T^i_o, Q^i)).
\end{align}
Because the form of $g^i$ is known, we use $g^i(\mathbf{\epsilon}^{i*}, \beta_r^i, T^{i*}_b, T^{i*}_o, \pi^{iu}_o, \pi^{iu}_b, s^i_t, s^i_{t+1})$ as an upper confidence bound estimate of $g^i(\mathbf{\epsilon}^i, \beta_r^i, T^i_b, T^{i}_o, \pi^i_o, \pi^i_b, s^i_t, s^i_{t+1})$.

We now have a full loop of inference and belief update. In the inference stage, each agent running our algorithm would keep a GP for $r$ and a corresponding estimated $Q$-function distribution. The agent would also keep track of policy parameters for each of the other agents. Given the expression above, the lower confidence bound of not entering any jointly unsafe states for our agent is
\begin{equation}
\begin{split}
    &lc(s^1_t, a^1_t) \! = \! \int\displaylimits_{\mathcal{S}} \! \left(1 \! - \! \int\displaylimits_{u \in U: u^1 = \varsigma} \! \prod_{i=2}^N g^i(\dots)du\right) \! f(\varsigma | s^1_t,\! a^1_t) d\varsigma,\\
\end{split}
\end{equation}
where $g^i(\dots)$ is short for $g^i(\mathbf{\epsilon}^i, \beta_r^i, T^i_o, T^i_b, \pi^{iu}_o, \pi^{iu}_b, s^i_t, u^i)$.

\subsection{Overall Algorithm}
Algorithm~\ref{safemdp_alg} introduces the overall method. We compute the expected lower bound reward according to Eq.~\eqref{return}, and then select the set of individually safe actions $A_{\text{hi-rew}}$ (lines 6-7).
We then compute the set of actions, $A_{\text{joint-safe}}$, with low risk of joint unsafety as defined in Definition~\ref{safe_def} of Sec.~\ref{agent_model} (lines 8-9), and
also compute the set of actions, $A_{\text{safe}}$ that satisfies both constraints (line 10).
Finally the algorithm selects the action with the lowest probability of joint unsafety if the available action set is empty, or selects the action that optimizes the agent's objective $Obj$ within the available action set (line 11-13).
We update the Q-function and the parameters following Eq.~\eqref{V_update} and Sec.~\ref{parameter_sec} (line 15-17).

\begin{figure}[h]
	\vspace{-10px}
	\begin{minipage}{\columnwidth}
		\begin{algorithm}[H]
			\caption{Multi-agent Safe $Q$-Learning}\label{safemdp_alg}
			\begin{algorithmic}[1]
				\STATE {\bfseries Input:} $\mathcal{S}, \mathcal{A}, f, S_0, c, h, \tau, \beta_r, Obj$
				\STATE Initialize $Q$.
				\STATE Initialize $\mathcal{GP}_r$ for reward estimate.
				\STATE Initialize $\bm{\epsilon}$, $\mathbf{T_b}$ and $\mathbf{T_o}$.
				\FOR{$t = 1, 2, ...$}
				\STATE Compute $lr(s^i_t, a)$ and $return(s^i_t, a)$ for $\forall a \in \mathcal{A}$
				\STATE $A_{\text{hi-rew}}\!\gets\!\{a | lr(s^i_t, a)\!\geq\!h \land return(s^i_t, a)\!>\!\tau\}$
				\STATE Compute $lc(s^i_t, a), \forall a \in \mathcal{A}$
				\STATE $A_{\text{joint-safe}} \leftarrow \{a | lc(s^i_t, a) \geq c \}$
				\STATE $A_{\text{safe}} \leftarrow A_{\text{hi-rew}} \bigcap A_{\text{joint-safe}}$
				\IF {$A_{\text{safe}} = \emptyset$}
				\STATE $A_{\text{safe}} \gets \arg\max_{a} lc(s^i_t, a)$
				\ENDIF
				\STATE $a^i_{t} \leftarrow \arg\max_{a\in A_{\text{safe}}}Obj(s^i_t, a)$
				\STATE $s^i_{t+1} \sim f(s^i_{t}, a^i_{t})$
				\STATE Update $\mathcal{GP}_{r}$ using $r(s^i_{t+1})$.
				\STATE Update $Q$ with $\mathcal{GP}_{r}$ and $f$
				\STATE Update $\bm{\epsilon}$, $\mathbf{T_b}$ and $\mathbf{T_o}$
				\ENDFOR
			\end{algorithmic}
		\end{algorithm}
	\end{minipage}
	\vspace{-10px}
\end{figure}

\section{Theoretical Results} \label{sec:theory}
In this section, we discuss the theoretical results of our algorithm. Mainly, we discuss the accuracy of our GP estimate on the reward along with $\beta_r(t)$, on the value function and we discuss the computational complexity of our algorithm.

\noindent\textbf{Reward Estimation Accuracy with Gaussian Processes.} The confidence interval of the GP for $r$ depends on $\beta^i_r(t)$, whose tuning has been well studied in \cite{turchetta2016safe} for the single-agent \textsc{SafeMDP} algorithm. Their result can be applied to our setting by choosing
\begin{align}\label{beta_choice}
    \beta_r^i(t) = 2B^i + 300\alpha^i_t\log^3(t/\delta^i),
\end{align}
where $B^i$ is the bound on the RKHS norm of the function $r(\cdot)$, $\delta^i$ is the probability of agent $i$ visiting individually unsafe states, and $\alpha^i_t$ is the maximum mutual information that can be gained about $r(\cdot)$ from $t$ noisy observations. The information capacity $\alpha^i_t$ has a sublinear dependency on $t$ for many commonly used kernels \cite{srinivas2010gaussian}. Assuming $||r||^2_k\leq B^i$ and the noise $w_t$ is zero-mean conditioned on the history as well as uniformly bounded by $\eta$ for all $t > 0$, if we choose $\beta^i_t$ above, then for all $s\in \mathcal{S}$ and $t>0$, with probability at least $1 - \delta^i$ that $r(s^i) \in {C_r}_t(s^i)$, where ${C_r}_t(s^i)$ is the estimated confidence bounds of the reward function on state $s^i$ at time step $t$ \cite{turchetta2016safe}.

\noindent\textbf{Value Function Estimation Accuracy.} To be able to accurately estimate the Boltzmann policy, the algorithm must make accurate estimates of the value function. Based on the reward functions' estimation accuracy, we can derive the following two theorems for the accuracy of value functions.
\vspace{-2px}
\begin{theorem}\label{value_func_prob}
\textit{If 1) $\beta^i_r(t)$ follows Eq.~\eqref{beta_choice}, 2) the states visited by agent $i$ are also visited by the estimating agent, and 3) $Q$ is in the form of universal state representation, then $Q(s, a) \in C_{Q_t}(s, a)$ with at least probability $1 - \delta^i$.}
\end{theorem}
\begin{proof}
	\cite{turchetta2016safe} proved that with the choice of $\beta^i_r(t)$ above, there is at least probability $1 - \delta^i$, $r(s^i) \in {C_{r_t}}(s^i), \forall s^i \in \mathcal{S}$. The Q-learning algorithm we define is by sampling potential transitions (to $s'$) and perform
	\begin{equation}
	Q_t(s^i, a^i) = r_t(s') + \gamma \max_{a'}Q_t(s', a')
	\end{equation}
	Hence at convergence with a universal representation of Q-function, the Q-function can be written as
	\begin{equation}
	Q_t(s^i, a^i) =   \mathbb{E}[r_t(s')] + \gamma \mathbb{E}[r_t(s'')] + \gamma^2 \mathbb{E}[r_t(s''')] + \dots
	\label{Q_expansion}
	\end{equation}
	where we use $s'$, $s''$, $s'''$, $\dots$ denote the future states in the optimal trajectory of taking $a^i$ at $s^i$. Equation~\eqref{Q_expansion} holds for each of $Q_t^\mu, Q_t^l$, and $Q_t^u$ respectively with $r_t^\mu, r_t^l$, and $r_t^u$. Since $\forall s \in S$, $r_t(s) \in C_{r_t}(s)$ with probability at least $1-\delta^i$; we have $\mathbb{E}[r_t(s')] \in [\mathbb{E}[r_t^l(s')], \mathbb{E}[r_t^u(s')]]$, with probability at least $1-\delta^i$ for any state distribution. Therefore with probability at least $1-\delta^i$,
	\begin{equation}\label{value_func_var}
	\begin{split}
	Q_t(s^i, a^i) \in &[\mathbb{E}[r_t^l(s')] + \gamma\mathbb{E}[r_t^l(s'')] + \dots,\\ 
	& \mathbb{E}[r_t^u(s')] + \gamma\mathbb{E}[r_t^u(s'')]  + \dots]\\
	= &[Q^l_t(s^i, a^i), Q^u_t(s^i, a^i)].\qquad \qedhere
	\end{split}
	\end{equation}
\end{proof} 
\vspace{-6px}

Even without the prior knowledge on $\delta^i$, ${C_Q}_t(s^i, a^i)$ covers at least $\beta^i_r(t)$ standard deviations from the mean value, where the probability of $Q(s^i, a^i)$ being bounded by ${C_Q}_t(s^i)$ can be directly found using a standard Z-score table.

However, from Eq.~\eqref{V_update}, we observe the variance of the value function is monotonically increasing during the update. Therefore, we need to bound the confidence intervals of the value function. The next theorem establishes such a bound.

\begin{theorem}\label{value_func_conf}
$Q^u_t(s, a) - Q^l_t(s, a) \leq \frac{2\beta_r(t)\max_{s \in S} \sigma_{r_t}(s)}{1-\gamma}$ for discounted infinite-horizon MDP, $\forall s \in \mathcal{S}, a\in\mathcal{A}$ and any given $\beta_r(t)$. 
\end{theorem}
\begin{proof}
	Using Eq.~\eqref{Q_expansion},	
	\begin{equation}\label{value_func_var2}
	\begin{split}
	&Q^u_t(s, a) - Q^l_t(s, a) \\
	& =( \mathbb{E}[r^u(s')] + \gamma \mathbb{E}[r^u(s'')] + \gamma^2 \mathbb{E}[r^u(s''')] + \dots ) \\
	& \text{\space \space}- ( \mathbb{E}[r^l(s')] + \gamma \mathbb{E}[r^l(s'')] + \gamma^2 \mathbb{E}[r^l(s''')] + \dots)\\
	& = ( \mathbb{E}[r^u(s') - r^l(s')] + \gamma \mathbb{E}[r^u(s'') - r^l(s'')] + \dots \\
	& \leq 2\beta_r(t) \mathbb{E}[\sigma_{r_t}(s')]) + 2\gamma\beta_r(t) \mathbb{E}[\sigma_{r_t}(s'')] + \dots \\
	& = 2\beta_r(t) (\mathbb{E}[\sigma_{r_t}(s')] + \gamma \mathbb{E}[\sigma_{r_t}(s'')] + \dots\\
	& \leq \frac{2\beta_r(t)\max_s \sigma_{r_t}(s)}{1-\gamma}.\qquad\qquad\qquad\qquad\qquad \qedhere
	\end{split}
	\end{equation}
\end{proof} 
\vspace{-4px}

\noindent\textbf{Computational Complexity.} For the algorithm to have online execution ability, it must be scalable and fast. With the prior knowledge that GP updates are $O(|\mathcal{S}_{\textrm{v}}|^3)$ where $\mathcal{S}_{v}$ is the set of states that are visited by any agent \cite{rasmussen2004gaussian}, we conclude our algorithm is linear in $|\mathcal{A}|$ and polynomial in $|\mathcal{S}_{v}|$.

\begin{theorem}
\textit{The Multi-agent Safe $Q$-Learning algorithm has a time complexity of $O(|\mathcal{S}_{v}|^3 + |\mathcal{A}| + Nt)$ at time step $t$ in a MDP-MA with $N$ agents using classic GPs. Using GP kernel approximation methods described in \cite{dasgupta2018finite, banerjee2012efficient}, the time complexity reduces to $O(|\mathcal{S}_{v}| + |\mathcal{A}| + Nt)$. If we apply sampling methods described in \cite{jain2018learning}, we can further reduce the time complexity to $O(|\mathcal{A}| + Nt)$.}
\end{theorem}
\begin{proof}
	With the prior knowledge that GP updates are $O(|\mathcal{S}_v|^3)$ \cite{rasmussen2004gaussian}, we can derive the time complexity bounds as follows.
	
	There are four major components of this algorithm.
	GP update for the reward function is on the state space and for all agents, hence complexity $O(|\mathcal{S}_v|^3)$. Choosing the optimal action among the safe actions has a time complexity of $O(|\mathcal{A}|)$. $Q$-function update with sampling takes $O(1)$. Optimization for $\epsilon$ and $T^i_b$ is linear to the number of steps in trajectory, hence complexity $O(Nt)$.
	
	Incorporating all complexity bounds above, the time complexity bound of the algorithm is $O(|\mathcal{S}_v|^3 + |\mathcal{A}| + Nt)$. 
	
	With optimizations on GPs in \cite{dasgupta2018finite, banerjee2012efficient} to make linear complexity and sampling in \cite{jain2018learning} to make constant complexity, the algorithm complexity becomes $O(|\mathcal{S}_v| + |\mathcal{A}| + Nt)$ and $O(|\mathcal{A}| + Nt)$ respectively.
	 \qedhere
\end{proof} 
\vspace{-10px}

\section{Experiments} \label{sec:experiments}
To assess the performance of our method, \textbf{Multi Safe Q-Agent}, we performed several simulation experiments using the following baselines:
\begin{itemize}[nosep]
	\item \textbf{Single Safe MDP Agent:} The algorithm in \cite{turchetta2016safe} with the modifications discussed in Section~\ref{iterative_safemdp}
	\item \textbf{Na\"ive Q-Agent:} $Q$-learning agent that assumes all other agents apply uniformly random policies.
	\item \textbf{Bayesian Q-Agent:} $Q$-learning agent receives additional negative reward for entering unsafe states and keeps a Bayesian belief on others' policies on discretized domain.
	\item \textbf{Safe Q-Agent:} $Q$-learning agent that takes all agents' states as its state space and receives additional negative reward for both individual unsafety and joint unsafety.
\end{itemize}

The GPs employed by the agents use an RBF Kernel with a length scale of $10$ and a prior standard deviation of $10$ and a White Noise Kernel with a prior standard deviation of $10$. The Q-functions are estimated using neural networks with two hidden layers of size $50$.


\subsection{Mars Rover Experiment} \label{mars_versus}
\begin{figure*}[t]
	\centering
	\includegraphics[width=0.75\textwidth]{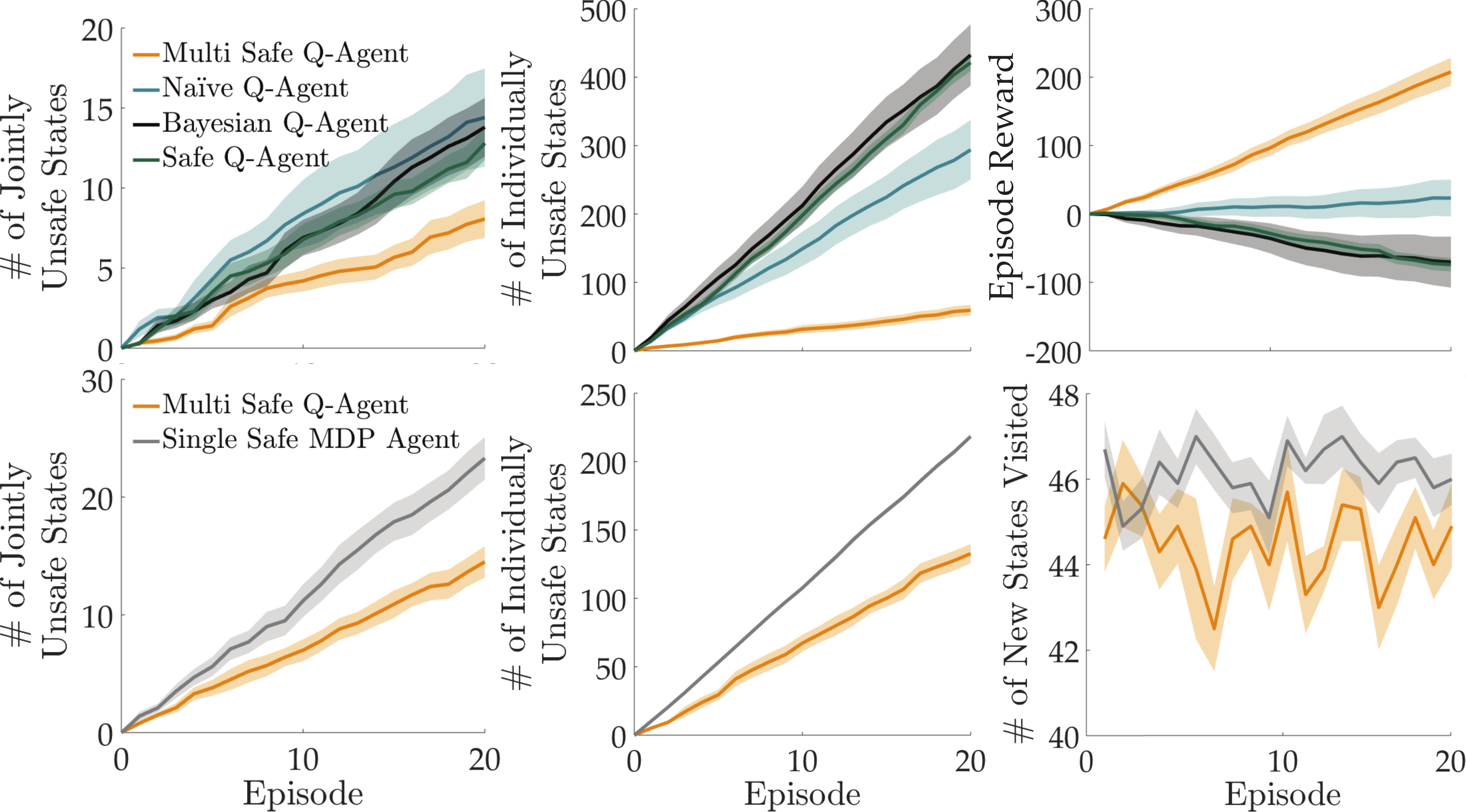}
	\caption{Experiment Results for Exploitation (top) and Exploration (bottom) with Mars Rovers (mean$\pm$s.e.)}
	\vspace{-7px}
	\label{fig.rover_combined}
\end{figure*}

\begin{figure*}[h]
	\centering
	\includegraphics[width=0.8\textwidth]{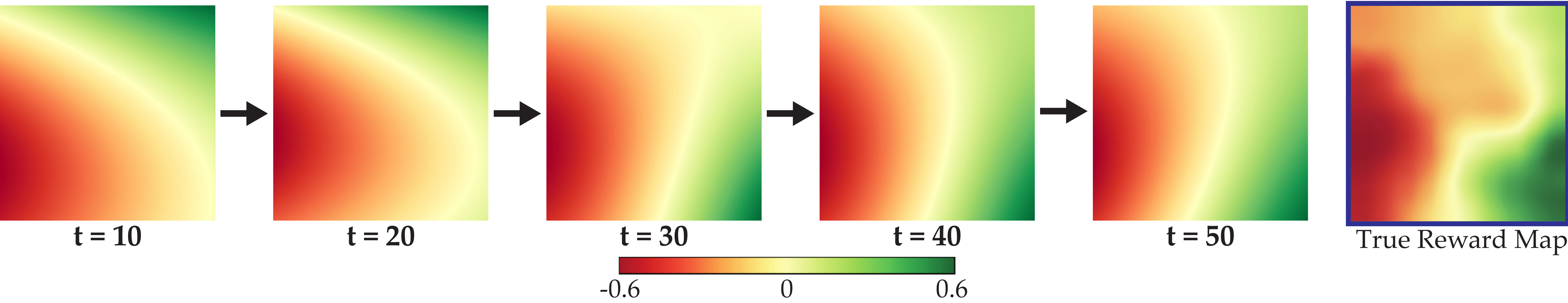}
	\caption{Learned reward values by the Multi Safe Q-Agent when exploring, and the true reward map are shown.}
	\vspace{-7px}
	\label{fig.state_learning}
\end{figure*}

We downloaded the Mars surface map from \href{http://www.uahirise.org/PDS/DTM/PSP/ORB_010200_010299}{High Resolution Imaging Science Experiment (HiRISE)} \cite{mcewen2007mars}, and selected a square region with an area of $100\textrm{ m}^2$, starting at $30.6^\circ$ latitude and $202.2^\circ$ longitude. The agents have $4$ possible actions: up, down, left and right. Each action moves the agent $1\textrm{ m}$ with some error that follows independent Gaussian distributions in both axes with means $0$ and variances $0.1$. When two rovers are too close to each other, they collide. These define joint safety. If an agent takes an action towards outside the boundary, it respawns at the other side of the map so that an agent is always forced to move instead of trying to stay at its original state to avoid collision. The individual safety condition is the altitude: the rover may be not recoverable when its altitude is too low. 

We simulate two objectives: Exploitation and exploration. In each, there exist $19$ $\epsilon$-greedy Q-learning agents with randomly chosen $\epsilon$ values. They do not try to avoid unsafe states. The experiment is set up in this way so that probability of collisions in the environment is high without careful navigation. We test how each agent performs in such hostile environment with its own objective. Agents adopt $h\!=\!-0.5$, $\tau\!=\!1$, $c\!=\!0.7$ and collision distance threshold of $0.1$ meters.

We run each experiment $10$ times with random initial states (constrained to individually and jointly safe states). Each experiment is run with 20 episodes of 50 time steps. We use Na\"ive Q-Agent, Bayesian Q-Agent and Safe Q-Agent as baselines for the exploitation setting because they optimize episode reward without a specific goal of optimizing for exploration. We use Single Safe MDP Agent as baseline for exploration because it optimizes only for state exploration.

\noindent\textbf{Exploitation.} In this experiment, we use altitude as the agents' reward. The results are shown in Fig.~\ref{fig.rover_combined}(top). Multi Safe Q-Agent significantly outperforms other candidates in both safety and episode reward (we exclude the additional negative reward for Bayesian Q-Agent and Safe Q-Agent for fairness).

Bayesian Q-Agent and Safe Q-Agent are safer than the Na\"ive Q-Agent, but since their rewards are corrupted by the additional penalty, they did not properly learn how to maximize reward. Na\"ive Q-Agent's episode reward is also low due to the limit in the actions it can choose. Since it assumes random policy for other agents, it unnecessarily eliminates many trajectories that would possibly lead to higher episode reward. On the other hand, Multi Safe Q-Agent outperforms the baselines by getting higher and higher rewards after learning from previous episodes.

\noindent\textbf{Exploration.} In this experiment, Multi Safe Q-Agent chooses the actions to visit the most uncertain state (the state with highest variance by the GP) that satisfies the safety constraints in the algorithm. We compare its performance against Single Safe MDP Agent, which is already designed for exploration. We show our agent's estimate of the altitude map in Fig.~\ref{fig.state_learning}. Quantitative results are in Fig.~\ref{fig.rover_combined}(bottom). Multi Safe Q-Agent significantly outperforms in safety, but is marginally worse in terms of the number of new states visited. This is reasonable, because achieving higher safety usually means a more constrained set of actions, which then harms the exploration. Hence, we conclude it is much safer and performs exploration comparably well.

\subsection{Quadcopter Collaborative Experiment}
In this experiment, we aim to evaluate agents' capability of safe collaboration\footnote{A video of the experiment is at \url{http://youtu.be/l76glwgF67k}.}. Two quadcopters are initialized in the domain to ship an item to the destination together. The item would be lost if the quadcopters are too far apart from each other. The setup is shown in Fig.~\ref{fig.setup}.

\begin{figure}[h]
	\centering
	\includegraphics[width=0.6\columnwidth]{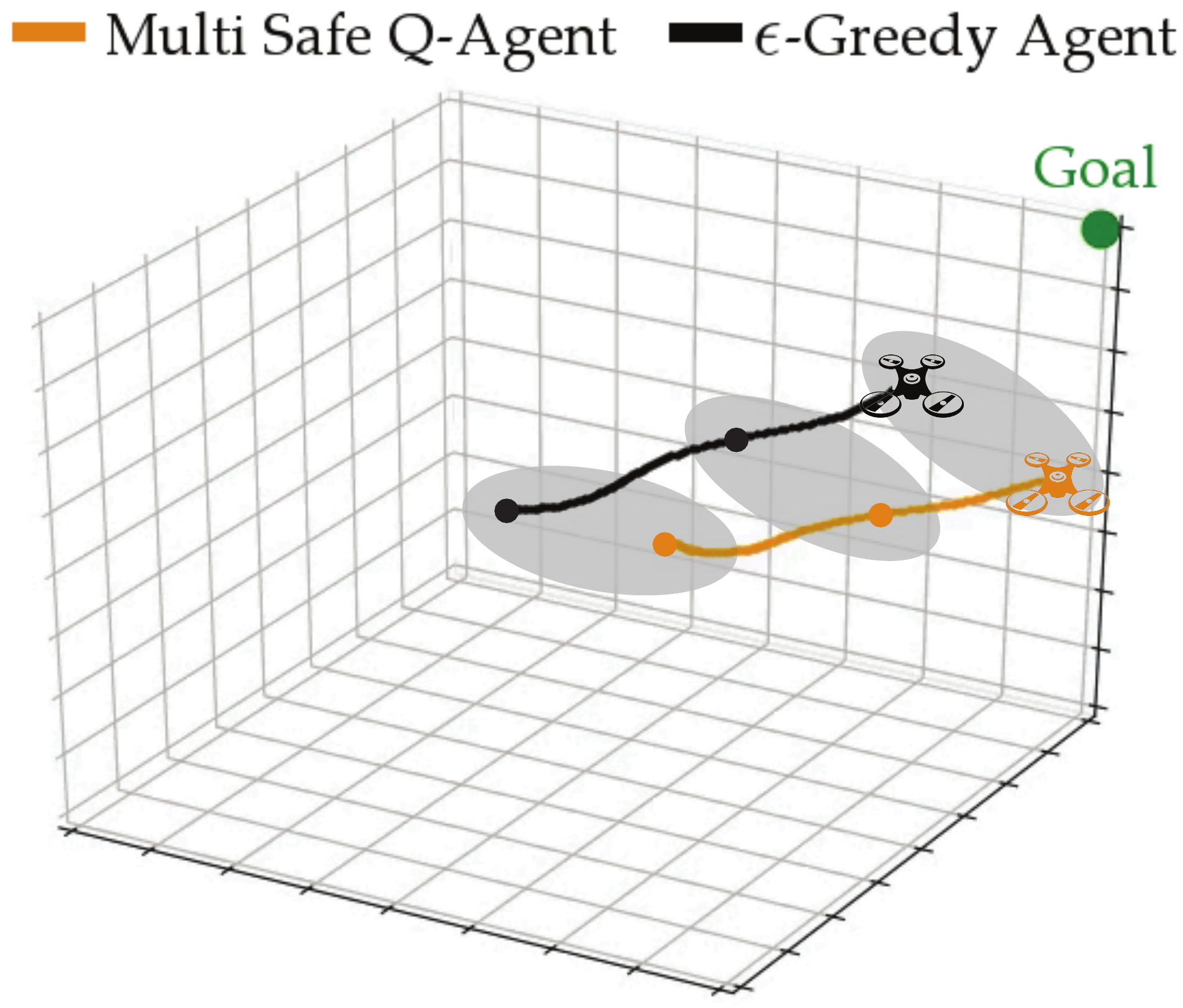}
	\caption{Quadcopter Experiment Setup}
	\vspace{-8px}
	\label{fig.setup}
\end{figure}
We discretize the action space of the quadcopters such that each has $6$ possible actions: up, down, forward, backward, left and right, each of which would move the agent $0.1$ unit. When the agents move, their actions have independent Gaussian errors in all three axes with means $0$ and variances $0.1$. We discretize the time such that when an action is chosen, a PID controller navigates the quadcopter to the destination before the next time step. Quadcopters start at $(0.5, 0, 0)$ and $(-0.5, 0, 0)$. The destination is at $(2, 2, 2)$. The maximum safe distance between the agents is $3$. Reward is the normalized negation of the Euclidean distance to the destination. $h = -8.0, \tau = 1, c = 0.7$. This task is extremely challenging as the quadcopters are uncoordinated and do not have any information about each other. We run 10 independent experiments in this domain. Each experiment consists of $100$ episodes. Each episode terminates when any safety constraint is violated or when the number of time steps reaches $100$. We compare our method with Safe Q-Agent, as it is the only baseline that specifically accounts for both individual and joint safeties. Each experiment is initialized with one of the agents of interest, and an $\epsilon$-greedy Q-learning agent with $\epsilon=0.1$.

\begin{figure}[t]
	\centering
	\includegraphics[width=\columnwidth]{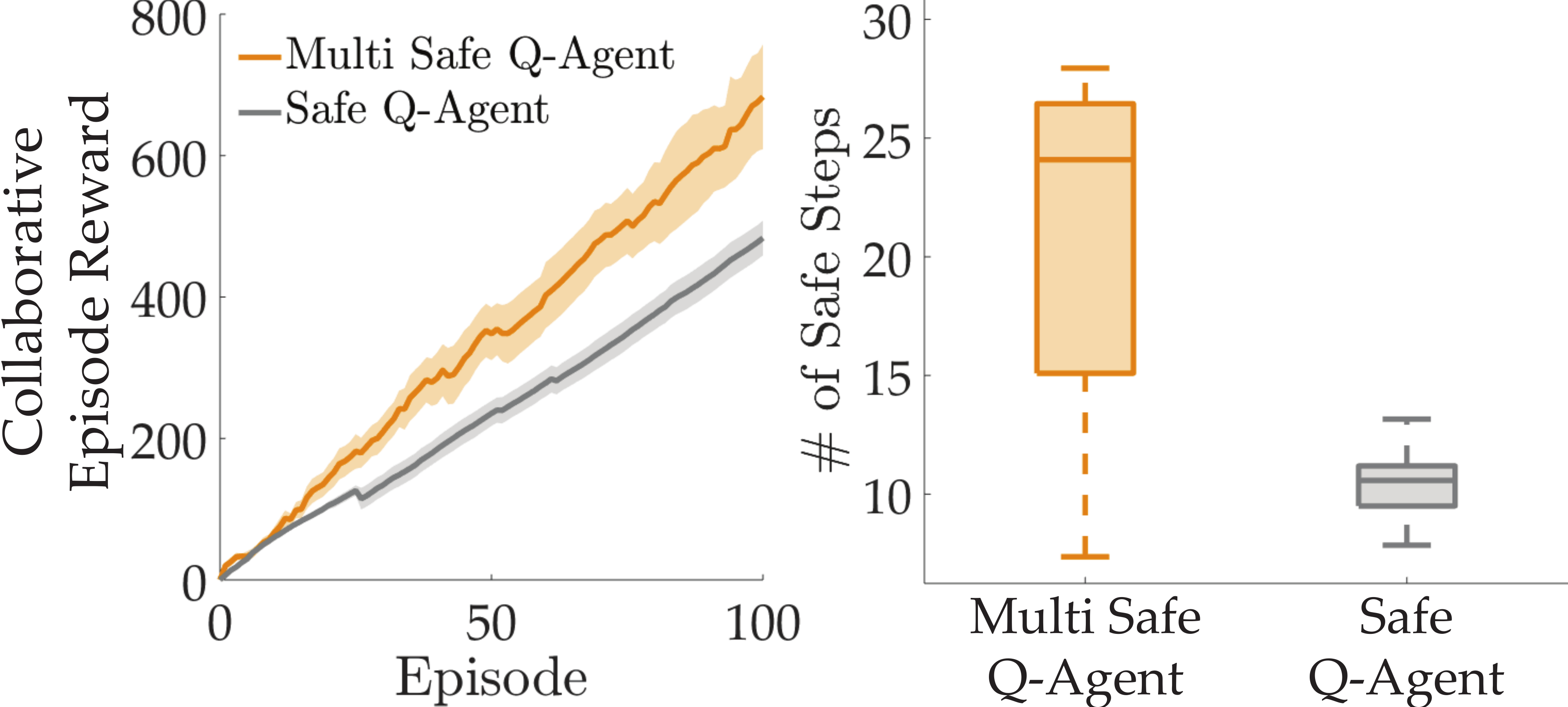}
	\vspace{-13px}
	\caption{Quadcopter Experiment Results (mean$\pm$s.e. for the line plots). The box plot shows the distribution of average number of safe steps across $10$ experiment runs before termination due to unsafety. Each experiment's number of safe steps is averaged over $100$ episodes.}
	\vspace{-15px}
	\label{fig.quadcopter}
\end{figure}
Figure~\ref{fig.quadcopter} shows the total episode reward of the agents and also the number of safe steps before the task is failed. While both agents collect higher rewards by learning from previous episodes, Multi Safe Q-Agent clearly outperforms the Safe Q-Agent in both total reward and safety.



\section{Discussion}
\textbf{Summary.} In this paper, we presented a decentralized planning algorithm that enables agents to avoid stepping into \emph{individually} or \emph{jointly} unsafe states. We use a GP-based approach to estimate safety and uncertainty. Our algorithm assumes very little prior knowledge on other agents and learns their policies through observations. We showed the algorithm has polynomial time complexity in the number of visited states, available actions, and agents. We also gave a performance guarantee on the estimated $Q$-functions of other agents. We empirically demonstrate our algorithm in collaborative Mars rover and quadcopter experiments. Our results suggest our algorithm outperforms all other baselines in terms of safety and also has the best performance in the exploitation setting. 
We would like to emphasize that although there is extensive work in the area of decentralized planning and control, most previous multi-agent safe learning work focuses either on precomputed policies or is not decentralized \cite{julian2018deep, vrohidis2017safe}. Hence, a fair comparison would be only with algorithms related to our problem such as SafeMDP \cite{turchetta2016safe}.

We would also like to note that our work extends to human-robot collaboration tasks where each human can be modeled in the same decentralized way. For example, in underwater robotics \cite{khatib2016ocean}, robots assist scuba divers to complete the tasks where they share the same reward and can use the diver's vital failure, e.g. running out of oxygen, as the joint unsafety. Another example is that surgical robots can share the reward with surgeons and use the patient's condition to define joint unsafety. This method also has the potential to go beyond robotics and can be applicable in other multi-agent AI settings such as the game of Dota and League of Legends, where a team of agents would share a common reward and the joint unsafety would be the home being attacked by the opponents.

\noindent\textbf{Limitations and Future Work.}
There are a few limitations that we plan to address as part of future work. The algorithm only avoids immediate jointly unsafe states, but does not plan on other agents' potential trajectories. One can easily imagine a scenario where the joint unsafe state is guaranteed to occur multiple steps ahead. A potential opportunity here is to have a multi-step trajectory roll out and select the safest option. Another strong assumption we make is that reward function is shared across agents. This assumption allows for adequate online learning and quick reaction from our agent. However, in some of the real-world collaborative tasks, this assumption may not hold. Two potential solutions are to directly learn other agents policies or rewards. Third, one could improve the modeling of other agents' policies by using different learning algorithms rather than GPs. Finally, an empirical validation of our algorithm on real-world experiments is required.
 
\noindent\textbf{Conclusion.} Despite these limitations, we are encouraged to see our algorithm demonstrating safe behavior in collaboration with other agents with very little prior knowledge of their policies. We also look forward to exploring applications of our algorithm beyond collaborative navigation to other multi-agent partially observable settings such as in manipulation or in human-robot interaction.


\section*{Acknowledgments}
Toyota Research Institute ("TRI")  provided funds to assist the authors with their research but this article solely reflects the opinions and conclusions of its authors and not TRI or any other Toyota entity. We acknowledge funding by FLI grant RFP2-000 and the DARPA Assured Autonomy program.



%

\bibliographystyle{unsrtnat}
\bibliography{ref}

\end{document}